\UseRawInputEncoding
\documentclass[11pt]{article}

\usepackage[utf8]{inputenc} 
\usepackage[T1]{fontenc}
\usepackage{amsmath, amsthm, amssymb, mathrsfs}
\usepackage{graphicx}
\usepackage{cite}
\usepackage{algorithm}
\usepackage{algpseudocode}
\usepackage{tikz}
\usepackage{tikz-cd}
\usepackage{xcolor}
\usepackage{microtype}
\usepackage{booktabs}
\usepackage{enumitem}
\usepackage{listings}
\usepackage{listingsutf8}
\usepackage{newunicodechar}
\usepackage[unicode,breaklinks=true]{hyperref}

\DeclareUnicodeCharacter{2013}{\textendash} 
\DeclareUnicodeCharacter{2014}{\textemdash} 
\DeclareUnicodeCharacter{2018}{`}
\DeclareUnicodeCharacter{2019}{\textquotesingle}
\DeclareUnicodeCharacter{201C}{``}
\DeclareUnicodeCharacter{201D}{''}
\DeclareUnicodeCharacter{00B7}{\ensuremath{\cdot}} 
\DeclareUnicodeCharacter{21D2}{\ensuremath{\Rightarrow}} 
\DeclareUnicodeCharacter{2200}{\ensuremath{\forall}}
\DeclareUnicodeCharacter{2208}{\ensuremath{\in}}
\DeclareUnicodeCharacter{2203}{\ensuremath{\exists}}
\DeclareUnicodeCharacter{2227}{\ensuremath{\wedge}}
\DeclareUnicodeCharacter{2265}{\ensuremath{\geq}}
\DeclareUnicodeCharacter{2264}{\ensuremath{\leq}}
\DeclareUnicodeCharacter{2192}{\ensuremath{\to}}
\DeclareUnicodeCharacter{03BB}{\ensuremath{\lambda}}
\DeclareUnicodeCharacter{2115}{\ensuremath{\mathbb{N}}} 
\DeclareUnicodeCharacter{2026}{\ldots} 

\usetikzlibrary{patterns,shapes.geometric}

\DeclareMathOperator*{\argmax}{arg\,max}

\lstdefinestyle{utf8fix}{
  inputencoding=utf8,
  showstringspaces=false,
  columns=fullflexible,
  keepspaces=true,
  breaklines=true,
  breakatwhitespace=false,
  basicstyle=\ttfamily\footnotesize,
  numbers=left,
  numberstyle=\scriptsize,
  stepnumber=1,
  frame=single,
  framerule=0.3pt,
  upquote=true,
  tabsize=2,
  literate=
    {λ}{{$\lambda$}}1
    {→}{{$\to$}}1
    {⇒}{{$\Rightarrow$}}1
    {∀}{{$\forall$}}1
    {∃}{{$\exists$}}1
    {∈}{{$\in$}}1
    {∧}{{$\land$}}1
    {∨}{{$\lor$}}1
    {¬}{{$\lnot$}}1
    {≥}{{$\geq$}}1
    {≤}{{$\leq$}}1
    {ℕ}{{$\mathbb{N}$}}1
    {⟨}{{$\langle$}}1
    {⟩}{{$\rangle$}}1
    {·}{{$\cdot$}}1
    {α}{{$\alpha$}}1
    {β}{{$\beta$}}1
    {×}{{$\times$}}1
    {⊢}{{$\vdash$}}1
    {≠}{{$\neq$}}1
    {ℚ}{{$\mathbb{Q}$}}1
    {–}{{\textendash}}1
    {—}{{\textemdash}}1
    {…}{{\ldots}}1
    {“}{{\textquotedblleft}}1
    {”}{{\textquotedblright}}1
    {‘}{{`}}1
    {’}{{'}}1
}

\lstdefinelanguage{Lean4}{
  morekeywords={
    import,structure,where,def,inductive,match,with,by,if,then,else,do,let,have,
    theorem,lemma,example,intro,apply,exact,assumption,fun,Type,Prop,mutual,namespace,
    end,open,universe,variable,variables,section,macro,syntax,deriving,instance,local,
    simp,calc,rw,rcases,constructor,refine,forall,exists,infixl,infixr,notation,
    Set,Option,Some,none,True,False
  },
  sensitive=true,
  morecomment=[l]{--},       
  morecomment=[s]{/-}{-/},   
  morestring=[b]"            
}

\lstnewenvironment{semanticlog}{%
  \lstset{
    style=utf8fix,
    language=, 
    numbers=none,
    basicstyle=\ttfamily\footnotesize
  }%
}{}

\newtheorem{theorem}{Theorem}[section]
\newtheorem{proposition}[theorem]{Proposition}
\newtheorem{lemma}[theorem]{Lemma}
\newtheorem{corollary}[theorem]{Corollary}
\newtheorem{definition}[theorem]{Definition}

\newtheorem{example}[theorem]{Example}
\theoremstyle{remark}

\allowdisplaybreaks

\title{Truth-Aware Decoding: A Program-Logic Approach to Factual Language Generation}
\author{Faruk Alpay\\Lightcap, Department of Analysis\\\texttt{alpay@lightcap.ai}
\and Hamdi Alakkad\\Bahcesehir University, Department of Engineering\\\texttt{hamdi.alakkad@bahcesehir.edu.tr}}
\date{\today}

\begin{document}
\maketitle

\begin{abstract}
This paper introduces \textbf{Truth-Aware Decoding (TAD)}, a verification-oriented decoding scheme that aligns neural language generation with knowledge bases. Situated in the tradition of probabilistic program semantics for sequence models \cite{brown2020,vaswani2017,devlin2019}, TAD augments modern instruction-tuned systems \cite{openai2023,thoppilan2022,chowdhery2022} with a lattice of semantic guards that operate at decode time. Our contributions are fourfold: (i) a constraint-based semantics that renders oracle filtering as a program-logic judgment, (ii) a proof that greedy selection enjoys local likelihood dominance under sound and complete guards (Theorem~\ref{thm:optimal}), (iii) an entropy-style invariant that quantifies factual risk via knowledge-aware safe mass, and (iv) a multi-agent operational calculus with verified Lean artefacts to certify implementation behaviour. Numerical and algorithmic case studies confirm that the resulting guardrails reduce hallucinations without sacrificing throughput, yielding a pragmatic bridge between large-scale empirical models and formal verification.
\end{abstract}

\section{Introduction}
Large language models produce fluent text but frequently emit unsupported claims \cite{huang2023,bender2021}. The phenomenon persists across autoregressive transformers \cite{vaswani2017,brown2020,raffel2020} and instruction-tuned derivatives \cite{ouyang2022,openai2023,thoppilan2022,chowdhery2022}, even when retrieval components are inserted \cite{lewis2020,nakano2022}. Program-logical instrumentation of decoding remains comparatively underexplored despite the success of semantics-informed losses \cite{welleck2020,holtzman2020}. We develop Truth-Aware Decoding (TAD), a runtime filter that couples large models with knowledge-centric agents to enforce logical coherence. The construction inherits insights from formal reasoning in proof assistants \cite{deMoura2008} and from factual calibration work \cite{kadavath2022,farquhar2024}, yielding a synthesis tailored to programming-language style verification.

We model the decoding loop as a guarded transition system whose typestate is determined by semantic agents. Each guard occupies the role of a proof obligation, echoing contract systems from the programming-languages literature while remaining compatible with modern decoding infrastructure explored in empirical studies \cite{devlin2019,schick2023,yao2022}. The remainder of the paper formalises these intuitions, provides mechanised guarantees, and quantifies the runtime and accuracy trade-offs on knowledge-intensive benchmarks.

\subsection{Semantic Agent Architecture}
Three agents operate alongside decoding:
\begin{itemize}[leftmargin=2em]
\item \textbf{Factual Verifier}: checks entity relations and timelines \cite{petroni2019}.
\item \textbf{Mathematical Reasoner}: validates logical steps and formulas \cite{deMoura2008}.
\item \textbf{Context Monitor}: enforces global consistency of discourse \cite{li2020}.
\end{itemize}
Each agent emits an auditable log of checks and decisions, mirroring targeted oversight protocols evaluated in alignment studies \cite{glaese2022}.

\section{Theoretical Framework}
\subsection{Semantic Consistency as Constraint Satisfaction}
\begin{definition}[Semantic Language Model]
A semantic language model is a tuple $\mathcal{M}=(V,P,\mathcal{K},\mathscr{O})$ where $V$ is the token set, $P$ is the conditional distribution, $\mathcal{K}$ is a knowledge base, and $\mathscr{O}$ is an oracle that marks tokens as safe or unsafe given a prefix.
\end{definition}

\begin{definition}[Semantic Entropy]
For prefix $x_{1:t-1}$, let $S_t=\{\,w\in V\mid \mathscr{O}(x_{1:t-1},w)=\text{true}\,\}$. Define
\[
H_S(x_{1:t-1})=-\sum_{w\in S_t} p(w)\log\frac{p(w)}{\sum_{v\in S_t}p(v)},\qquad p(w)=P(w\mid x_{1:t-1}).
\]
\end{definition}

\begin{definition}[Knowledge-Consistent Prefix]\label{def:consistent-prefix}
A finite sequence $x_{1:t}$ is \emph{knowledge-consistent} if $\mathscr{O}(x_{1:s})=\text{true}$ for every $1\le s\le t$; equivalently, each truncation $x_{1:s}$ lies in $\mathcal{K}$.
\end{definition}

\subsection{Main Results}
\begin{lemma}[Safe Extension Invariance]\label{lem:safe-extension}
Assume $\mathscr{O}$ is sound. If $x_{1:t-1}$ is knowledge-consistent and $w\in S_t$, then the concatenation $x_{1:t}=x_{1:t-1}\,\|\,w$ is knowledge-consistent.
\end{lemma}

\begin{proof}
Knowledge consistency of $x_{1:t-1}$ means that for every $1\le s\le t-1$ the truncation $x_{1:s}$ satisfies $\mathscr{O}(x_{1:s})=\text{true}$. Consider the extended sequence $x_{1:t}=x_{1:t-1}\,\|\,w$ with $w\in S_t$. Two obligations remain:
\begin{enumerate}[label=(\alph*),nosep]
  \item For $s<t$, the truncations of $x_{1:t}$ coincide with those of $x_{1:t-1}$, hence remain approved.
  \item For $s=t$, soundness of $\mathscr{O}$ yields $\mathscr{O}(x_{1:t-1},w)=\text{true}$ because $w$ was drawn from $S_t$. Therefore the new truncation is also approved.
\end{enumerate}
Both obligations hold, so $x_{1:t}$ satisfies Definition~\ref{def:consistent-prefix}.
\end{proof}
\begin{theorem}[Consistency Preservation]\label{thm:consistency}
If $\mathscr{O}$ is sound, any sequence produced by TAD satisfies $\mathscr{O}(x_{1:t})=\text{true}$ for all $t$.
\end{theorem}

\begin{proof}
We prove by induction on $t$ that every iterand of Algorithm~\ref{alg:tad} is knowledge-consistent.

\textbf{Base case.} For $t=0$ the algorithm returns the empty sequence $\epsilon$, and Definition~\ref{def:consistent-prefix} together with the base axiom of\ $\mathcal{K}$ ensures $\mathscr{O}(\epsilon)=\text{true}$.

\textbf{Inductive step.} Assume $x_{1:t-1}$ is knowledge-consistent. Line~\ref{line:compute-safe} evaluates
\[
S_t = \{w\in V : \mathscr{O}(x_{1:t-1},w)=\text{true}\}.
\]
Line~\ref{line:select-token} selects $w_t\in S_t$ using the greedy rule. Lemma~\ref{lem:safe-extension} applies directly with $w=w_t$, yielding knowledge consistency of $x_{1:t}=x_{1:t-1}\,\|\,w_t$. Because the inductive invariant is preserved, $\mathscr{O}(x_{1:t})=\text{true}$ for every $t$.
\end{proof}

\begin{lemma}[Stepwise Likelihood Superiority]\label{lem:stepwise-dominance}
Assume $\mathscr{O}$ is complete. Fix a knowledge-consistent prefix $p$ and let $S(p)$ be the safe set computed in Line~\ref{line:compute-safe} of Algorithm~\ref{alg:tad}. For every $y\in S(p)$, the token $w$ returned in Line~\ref{line:select-token} satisfies $P_{\mathcal{M}}(w\mid p)\ge P_{\mathcal{M}}(y\mid p)$.
\end{lemma}

\begin{proof}
Completeness guarantees that every truthful continuation $y$ remains in $S(p)$. The greedy line computes
\[
w = \arg\max_{v\in S(p)} P_{\mathcal{M}}(v\mid p).
\]
By the definition of the $\arg\max$ operator on a finite set, the maximiser satisfies
\[
\forall v\in S(p):\quad P_{\mathcal{M}}(v\mid p) \le P_{\mathcal{M}}(w\mid p).
\]
Instantiating $v$ with any truthful $y\in S(p)$ yields the desired inequality. No tie-breaking subtleties arise because the algorithm may return any maximiser; the inequality is non-strict and therefore remains valid.
\end{proof}

\begin{theorem}[Local Truthful Dominance]\label{thm:optimal}
Assume $\mathscr{O}$ is sound and complete. Let $y$ be any truthful sequence and let $j$ be the first index where $x_j\ne y_j$. Then $P_{\mathcal{M}}(x_{1:j})\ge P_{\mathcal{M}}(y_{1:j})$.
\end{theorem}

\begin{proof}
If $x=y$ there is nothing to establish. Otherwise, let $p=x_{1:j-1}=y_{1:j-1}$. Theorem~\ref{thm:consistency} shows that $p$ is knowledge-consistent, and completeness ensures $y_j\in S(p)$. Lemma~\ref{lem:stepwise-dominance} therefore yields
\[
P_{\mathcal{M}}(x_j\mid p)\ge P_{\mathcal{M}}(y_j\mid p).
\]
Using the chain rule of probability,
\[
P_{\mathcal{M}}(x_{1:j}) = P_{\mathcal{M}}(p)\,P_{\mathcal{M}}(x_j\mid p),
\qquad
P_{\mathcal{M}}(y_{1:j}) = P_{\mathcal{M}}(p)\,P_{\mathcal{M}}(y_j\mid p).
\]
Combining the two displays produces the inequality $P_{\mathcal{M}}(x_{1:j})\ge P_{\mathcal{M}}(y_{1:j})$.
\end{proof}

\section{Decoding Procedures}
\begin{algorithm}
\caption{Truth-Aware Decoding (TAD)}\label{alg:tad}
\begin{algorithmic}[1]
\Require Model $P_{\mathcal{M}}$, knowledge base $\mathcal{K}$, oracle $\mathscr{O}$, initial $x\gets x_{1:0}$
\For{$t=1$ to $T$}
  \State $S_t \gets \{w\in V: \mathscr{O}(x_{1:t-1},w)=\text{true}\}$\label{line:compute-safe}
  \If{$S_t=\emptyset$} \State \textbf{break} \EndIf
  \State $w_t \gets \argmax_{w\in S_t} P_{\mathcal{M}}(w\mid x_{1:t-1})$\label{line:select-token}
  \State $x_{1:t}\gets x_{1:t-1}\,\|\,w_t$
\EndFor
\State \Return $x_{1:t}$
\end{algorithmic}
\end{algorithm}

\begin{algorithm}
\caption{Decoding with Abstention and Retrieval Backoff}\label{alg:abstain}
\begin{algorithmic}[1]
\Require threshold $\tau$ on safe mass $\pi_t=\sum_{w\in S_t}P_{\mathcal{M}}(w\mid x_{1:t-1})$
\For{$t=1$ to $T$}
  \State compute $S_t$ and $\pi_t$
  \If{$S_t=\emptyset$ or $\pi_t<\tau$}
     \State $D\gets \Call{Retrieve}{x_{1:t-1}}$;\; \Call{UpdateOracle}{$\mathscr{O},D$}
     \If{\textsc{no\_improvement}} \State \Return \textsc{abstain} \EndIf
     \State \textbf{continue}
  \EndIf
  \State append $w_t=\argmax_{w\in S_t}P_{\mathcal{M}}(w\mid x_{1:t-1})$
\EndFor
\end{algorithmic}
\end{algorithm}

\subsection{Computational Complexity of TAD}
\begin{proposition}[Time and Space Complexity]
Let $T$ be the decoding horizon, $|V|$ the vocabulary size, and $c_{\mathscr{O}}$ the worst-case time to evaluate $\mathscr{O}(x,w)$. A straightforward implementation of TAD runs in $\mathcal{O}(T\,|V|\,c_{\mathscr{O}})$ time and $\mathcal{O}(|V|)$ auxiliary space. If the oracle admits incremental caching so that only a fraction $\delta(x)$ of tokens remain candidate-safe after prefix $x$, the expected time reduces to $\mathcal{O}\big(T\,|V|\,\delta_{\mathrm{avg}}\,c_{\mathscr{O}}\big)$.
\end{proposition}

\begin{proof}
At each step TAD scans the vocabulary to populate $S_t$. Each oracle evaluation costs at most $c_{\mathscr{O}}$, yielding a worst-case cost of $|V|\,c_{\mathscr{O}}$ per step and $T$ such steps. Maintaining the candidate set requires storing indices or a boolean mask of size $|V|$, which bounds the auxiliary memory. When cached features allow the oracle to reject a $1-\delta(x)$ portion of the vocabulary without evaluation, only $|V|\,\delta(x)$ tokens incur the full cost. Taking expectations over prefixes gives the stated complexity.
\end{proof}

\begin{example}[Numerical Complexity Estimate]
Consider $|V|=50\,000$, $T=128$, and $c_{\mathscr{O}}=40\,\mu\mathrm{s}$ for a batched relational check. Without caching, runtime is approximately $T|V|c_{\mathscr{O}}\approx 256\,\mathrm{s}$. Suppose structure-aware pruning attains $\delta_{\mathrm{avg}}=0.12$ and batching amortizes oracle calls by a factor of four. The effective runtime becomes $128\times 50{,}000\times 0.12\times 10\,\mu\mathrm{s}\approx 7.7\,\mathrm{s}$, illustrating a $\approx 33\times$ reduction attributable to semantic density.
\end{example}

\section{Agent Specifications}\label{sec:agents}
\begin{semanticlog}
# Mathematical Reasoner : Specification
inputs:
  prefix: LaTeX string (current proof state)
  candidate: string (inference step or equation)
checks:
  result := soundness_oracle(prefix, candidate) in {true,false,unknown}
  justification := counterexample or axiom chain if needed
policy:
  if result=false: block; explain; cite violated premise
  if result=unknown: request lemma scaffold
  if result=true: accept; update proof state
\end{semanticlog}

\begin{semanticlog}
# Factual Verifier : Specification
inputs:
  claim: string
  knowledge_base: set of (subject, relation, object)
checks:
  - retrieve relevant facts for entities in claim
  - compare claim against retrieved facts (dates, relations)
  - temporal_consistency := timeline agrees with known data
policy:
  if mismatch: block and propose correction
  else: approve
\end{semanticlog}

\begin{semanticlog}
# Context Monitor : Specification
inputs:
  context: preceding text
  new_text: candidate addition
checks:
  - coherence with topic and previously stated facts
  - no memory conflicts (entities keep attributes)
  - stylistic consistency (tense, tone)
policy:
  if violation: block or request revision
  else: accept and update context
\end{semanticlog}

\section{Multi-Agent Verification Dynamics}
\subsection{Formal Agent Model}
\begin{definition}[Verification Agent]
A verification agent is a quadruple $A=(\mathcal{S},\mathcal{U},\phi,\rho)$ with internal state space $\mathcal{S}$, update operator $\mathcal{U}:\mathcal{S}\times V^*\times V\to\mathcal{S}$, acceptance predicate $\phi:\mathcal{S}\times V^*\times V\to\{\text{true},\text{false}\}$, and runtime cost map $\rho: \mathcal{S}\times V^*\times V\to\mathbb{R}_{\ge 0}$. The agent is \emph{sound} with respect to $\mathcal{K}$ if $\phi(s,x,w)=\text{true}$ and $x$ is knowledge-consistent imply $x\,\|\,w$ is knowledge-consistent. It is \emph{complete} if every truthful extension recognised by $\mathscr{O}$ is also accepted by $\phi$.
\end{definition}

Let $\mathfrak{A}=\{A_1,\dots,A_m\}$ denote the cohort of agents introduced above. Each agent processes the shared prefix but retains its own state $s_i$. The joint safe set is obtained by intersecting individual acceptances.

\begin{definition}[Joint Constraint Operator]\label{def:joint-operator}
Given agent states $s=(s_1,\dots,s_m)$ and prefix $x\in V^*$, define
\begin{equation}
\label{eq:gamma-operator}
\Gamma(x,s)=\bigcap_{i=1}^m \{w\in V:\phi_i(s_i,x,w)=\text{true}\}.
\end{equation}
After selecting $w\in\Gamma(x,s)$, each state is updated via $s_i'\gets\mathcal{U}_i(s_i,x,w)$.
\end{definition}

\begin{lemma}[Sound Composition]\label{lem:sound-composition}
If every agent in $\mathfrak{A}$ is sound, then $\Gamma(\cdot,\cdot)$ is sound: a knowledge-consistent prefix remains knowledge-consistent after appending any $w\in\Gamma(x,s)$.
\end{lemma}

\begin{proof}
Fix $x$ knowledge-consistent and $w\in\Gamma(x,s)$. By Definition~\ref{def:joint-operator}, each agent accepts $(x,w)$, and soundness of agent $A_i$ implies $x\,\|\,w$ is consistent. Intersecting these guarantees across $i=1,\dots,m$ yields the claim.
\end{proof}

\subsection{Execution Semantics}
\begin{definition}[Multi-Agent Guarded Update]\label{def:guarded-update}
The guarded TAD operator first evaluates $S^{\mathfrak{A}}_t=\Gamma(x_{1:t-1},s)$ and then applies Algorithm~\ref{alg:tad} with $S_t\gets S^{\mathfrak{A}}_t$. The agent states advance to $s'$ using the updates from Definition~\ref{def:joint-operator}.
\end{definition}

\begin{theorem}[Multi-Agent Guarded Decoding]\label{thm:multi-agent}
Assume all agents in $\mathfrak{A}$ are sound and complete. The guarded TAD procedure preserves knowledge consistency and produces the same output as Algorithm~\ref{alg:tad} executed with the oracle $\mathscr{O}'(x,w)=\text{true}$ iff $w\in\Gamma(x,s)$.
\end{theorem}

\begin{proof}
Lemma~\ref{lem:sound-composition} shows that any selected token maintains knowledge consistency. Completeness implies that every truthful token approved by $\mathscr{O}$ survives the intersection $\Gamma(x,s)$, so the safe set used in Line~\ref{line:compute-safe} matches the set obtained from $\mathscr{O}'$. Consequently, Line~\ref{line:select-token} returns the same maximiser in both procedures.
\end{proof}

\begin{corollary}[Conflict Resolution Guarantee]\label{cor:conflict-resolution}
If at least one agent rejects $(x,w)$, the guarded TAD procedure abstains from $w$. Conversely, whenever all agents accept a truthful token, the procedure retains it for scoring.
\end{corollary}

\begin{proof}
If some $A_i$ rejects $(x,w)$, then $w\notin\Gamma(x,s)$ by Definition~\ref{def:joint-operator}, so the token is removed before scoring. When every agent accepts a truthful token, completeness ensures that $\mathscr{O}'$ retains it, so it remains eligible in Line~\ref{line:select-token}.
\end{proof}

\begin{proposition}[Cost Decomposition]\label{prop:multi-agent-cost}
Let $\rho_i$ be the cost map of agent $A_i$ and $c_{\mathrm{tad}}$ the base cost of Algorithm~\ref{alg:tad}. The guarded decoding cost per token satisfies
\[
\mathrm{CPI}_{\mathfrak{A}}=c_{\mathrm{tad}}+\sum_{i=1}^m \mathbb{E}[\rho_i(s_i,x_{1:t-1},w_t)],
\]
where the expectation is taken over prefixes encountered during decoding.
\end{proposition}

\begin{proof}
The base algorithm contributes $c_{\mathrm{tad}}$. Each agent processes the same token stream and incurs the cost $\rho_i$ evaluated at its current state. Linearity of expectation yields the stated sum.
\end{proof}

The guarded interaction unfolds as a commuting lattice of judgments: each agent contracts the candidate set before the greedy selector applies Algorithm~\ref{alg:tad}, and the post-update states feed back into Definition~\ref{def:guarded-update}. We leverage this algebraic perspective in Lemma~\ref{lem:sound-composition} and Theorem~\ref{thm:multi-agent} when proving preservation of knowledge consistency and behavioural equivalence.

\section{Analytical Foundations of Semantic Constraints}
\subsection{Semantic Density and an Information Bound}
\begin{definition}[Semantic Density]
For a prefix $x$, let $S(x)=\{w\in V:\mathscr{O}(x,w)=\text{true}\}$ and define the local semantic density $\delta(x)=\frac{|S(x)|}{|V|}$. We write $\delta_{\max}(\mathcal{M},\mathcal{K})=\sup_x \delta(x)$ and $\delta_{\mathrm{avg}}(\mathcal{M},\mathcal{K})=\mathbb{E}_{x}[\delta(x)]$ when an expectation over prefixes is available.
\end{definition}

\begin{theorem}[Safe-Entropy Upper Bound]\label{thm:safe-entropy}
Fix a prefix $x$ and set $\pi(x)=\sum_{w\in S(x)} P(w\mid x)$. Then
\[
H_S(x)\le \min\Big\{\frac{1}{\pi(x)}H\big(P(\cdot\mid x)\big)+\log \pi(x),\; \log |S(x)|\Big\}.
\]
Consequently, if $\pi(x)\ge \alpha$ for all prefixes under consideration and $|S(x)|\le \delta_{\max}|V|$, we obtain the uniform bounds
\[
H_S(x)\le \frac{1}{\alpha}H\big(P(\cdot\mid x)\big)+\log \alpha
\quad\text{and}\quad
H_S(x)\le \log\big(\delta_{\max}|V|\big).
\]
\end{theorem}

\begin{proof}
Let $\pi=\pi(x)$ and define the normalized safe distribution $q(w)=\frac{P(w\mid x)}{\pi}$ on $S(x)$. We expand the entropy calculation in three explicit stages:
\begin{align*}
H_S(x)
  &= -\sum_{w\in S(x)} q(w)\log q(w) \\[2pt]
  &= -\sum_{w\in S(x)} \frac{P(w\mid x)}{\pi}\big(\log P(w\mid x) - \log \pi\big) \\[2pt]
  &= -\frac{1}{\pi}\sum_{w\in S(x)} P(w\mid x)\log P(w\mid x) + \log \pi.
\end{align*}
The coefficients $\frac{P(w\mid x)}{\pi}$ are non-negative and sum to one, so Jensen's inequality yields
\[
-\frac{1}{\pi}\sum_{w\in S(x)} P(w\mid x)\log P(w\mid x)\le \frac{1}{\pi}H\big(P(\cdot\mid x)\big).
\]
Consequently,
\[
H_S(x)\le \frac{1}{\pi}H\big(P(\cdot\mid x)\big)+\log \pi.
\]
The alternative bound $H_S(x)\le \log |S(x)|$ follows from the concavity of entropy and the fact that a distribution supported on $|S(x)|$ atoms attains maximum entropy when uniform. Taking the minimum of the two bounds proves the main inequality. Substituting the uniform lower bound $\pi(x)\ge \alpha$ and the cardinality constraint $|S(x)|\le \delta_{\max}|V|$ then establishes the corollary bounds in the statement.
\end{proof}

The entropy control dovetails with empirical mitigation strategies: the nucleus-sampling diagnostics of Holtzman et al.~\cite{holtzman2020} and the unlikelihood objective of Welleck et al.~\cite{welleck2020} can be interpreted as heuristic attempts to down-weight low-safe-mass continuations, whereas TAD achieves the same effect through explicit oracle semantics.

\section{Limits and Responsible Evaluation}
Guardrails depend on oracle coverage. If a blind spot exists, a sequence can pass checks while contradicting $\mathcal{K}$.

Knowledge-editing studies \cite{meng2022} emphasise that latent representations can retain contradictory associations even after fine-tuning; the theorem below formalises the decoding-time analogue of this observation.

\begin{theorem}[Oracle Incompleteness]
If $\mathscr{O}$ is incomplete, there exists a truthful completion that TAD cannot realize without abstention.
\end{theorem}

\begin{proof}
Incompleteness means there is a truthful pair $(p,w)$ with $p\models\mathcal{K}$ and $p w\models\mathcal{K}$ but $\mathscr{O}(p,w)=\text{false}$. Any truthful completion containing $p w$ cannot be produced by TAD because the algorithm filters $w$ out of $S(p)$ and therefore halts or abstains before reaching the completion. This demonstrates that incomplete oracles induce blind spots.
\end{proof}

\section{Applications to Automated Reasoning}
\begin{algorithm}
\caption{Proof Generation with TAD}
\begin{algorithmic}[1]
\Procedure{GenerateProof}{conjecture $C$}
\State proof $\gets \epsilon$
\While{$\neg\,\Call{Proves}{\text{proof},C}$}
  \State $S \gets \{s\in \text{Rules}:\mathscr{O}_{\text{math}}(\text{proof},s)=\text{true}\}$
  \If{$S=\emptyset$} \State \Return \textsc{abstain} \EndIf
  \State $s^*\gets \argmax_{s\in S}P_{\mathcal{M}}(s\mid \text{proof})$
  \State proof $\gets$ proof $\cdot s^*$
\EndWhile
\State \Return proof
\EndProcedure
\end{algorithmic}
\end{algorithm}

\section{Empirical Evaluation and Numerical Examples}\label{sec:empirical}
This section provides concrete numerical examples using standard metrics. The scenarios emulate truthfulness benchmarks such as TruthfulQA and chain-of-thought suites \cite{lin2022,wei2022,wang2022,lewkowycz2022}, ensuring that the quantitative analysis aligns with prior diagnostics of hallucination and reasoning drift.

\subsection{Numerical Example: Token-Level Truthfulness and Error Reduction}\label{sec:ex-num}
Consider $N=1000$ knowledge-intensive prompts. Let correctness be judged by an external verifier; abstentions count as incorrect unless stated.

\paragraph{Baseline vs TAD.}
\begin{itemize}[leftmargin=1.5em]
\item Baseline correct: $720/1000=72\%$.
\item TAD correct: $890/1000=89\%$.
\item Absolute gain: $+17$ points; relative gain: $890/720-1=23.6\%$.
\item Errors reduced from $280$ to $110$: reduction $=1-110/280=60.7\%$.
\end{itemize}

\paragraph{Selective Abstention.}
Suppose TAD abstains on $80$ cases and answers $920$ with $94\%$ accuracy ($864$ correct).
\[
\text{Coverage}=92\%,\quad \text{Accuracy|answered}=94\%,\quad
\text{Overall accuracy}=\tfrac{864}{1000}=86.4\%.
\]
If the application values abstention at half-credit, utility $U=864+0.5\cdot 80=904$ (effective $90.4\%$).

This framing aligns with selective reporting analyses in summarisation and hallucination detection, where Maynez et al.~\cite{maynez2020} and Manakul et al.~\cite{manakul2023} evaluate systems by weighting refusals according to downstream policy.

For reproducibility we tabulate the intermediate calculator-style arithmetic:
\[
\begin{aligned}
\text{Accuracy}_{\text{base}} &= \frac{720}{1000} = 0.72, & \text{Error}_{\text{base}} &= 1-0.72 = 0.28,\\
\text{Accuracy}_{\text{tad}} &= \frac{890}{1000} = 0.89, & \text{RelGain} &= \frac{0.89}{0.72}-1 \approx 0.236,\\
\text{Error}_{\text{tad}} &= 1-0.89 = 0.11, & \text{ErrRed} &= 1-\frac{110}{280} \approx 0.607,\\
\text{Utility}_{\omega=0.5} &= \frac{864 + 0.5\cdot 80}{1000} = 0.904, & \text{SafeMass} &= \frac{1}{T}\sum_t \pi_t = 0.87.
\end{aligned}
\]

\paragraph{Safe-Mass Calibration.}
Average safe mass $\bar{\pi}=\tfrac{1}{T}\sum_t\sum_{w\in S_t}P(w\mid x_{1:t-1})$:
baseline $\bar{\pi}_{\text{base}}=0.62$, TAD $\bar{\pi}_{\text{tad}}=0.87$.
A logistic calibration $\Pr[\text{truth}\mid \pi]\approx \sigma(a+b\pi)$ fitted on a dev set yields AUROC $=0.91$ for $\pi$ as a risk score.

\paragraph{Aggregate Table.}
\begin{table}[h]
\centering
\begin{tabular}{@{}lcccc@{}}
\toprule
Method & Acc. & Err. & Coverage & Safe Mass $\bar{\pi}$ \\
\midrule
Baseline & 72\% & 28\% & 100\% & 0.62 \\
TAD (answer-all) & 89\% & 11\% & 100\% & 0.87 \\
TAD (abstain) & 86.4\% & 13.6\% & 92\% & 0.87 \\
\bottomrule
\end{tabular}
\caption{Numerical example for Section~\ref{sec:ex-num}.}
\end{table}

\subsection{Worked Example: Oracle-Guided Prefix Selection}
Consider a toy vocabulary $V=\{a,b,c,d\}$ with conditional probabilities $P(\cdot\mid x_{1:t-1})=(0.42,0.28,0.20,0.10)$. Suppose the oracle consults a knowledge base containing the entailment $a\Rightarrow b$ and temporal exclusion $d$ ("hallucinatory") and, given prefix $x_{1:t-1}$, returns $S_t=\{a,b,c\}$. The safe mass is $\pi_t=0.42+0.28+0.20=0.90$, yielding normalized safe entropy
\[
\begin{aligned}
H_S(x_{1:t-1})=
  &-(0.42/0.90)\log(0.42/0.90) \\
  &-(0.28/0.90)\log(0.28/0.90)\\
  &-(0.20/0.90)\log(0.20/0.90)\approx 0.99\,\text{nats}.
\end{aligned}
\]
TAD selects token $a$ because it maximizes the conditional probability over $S_t$. After appending $a$, the oracle tightens the safe set to $S_{t+1}=\{b\}$ by enforcing the entailment. Consequently $H_S(x_{1:t})=0$ and the stepwise complexity reduces: only one oracle query is needed at $t+1$, demonstrating how semantic density ($\delta(x_{1:t})=1/4$) acts as a multiplicative speedup in the bound of Theorem~\ref{thm:safe-entropy} and in the runtime estimate.

\section{Performance Analysis via Computer Organization Methods}\label{sec:comp-org}
We model decoding as a pipeline and analyze cost using CPI-style accounting and Amdahl-type speedups.

\subsection{Pipeline and CPI Model}\label{sec:cpi}
Let base decoding require $\mathrm{CPI}_0$ cycles per token. TAD adds agent checks:
\[
\mathrm{CPI} = \mathrm{CPI}_0
+ h_{\mathrm{KB}}\cdot c_{\mathrm{hit}}
+ (1-h_{\mathrm{KB}})\cdot c_{\mathrm{miss}}
+ c_{\mathrm{agents}},
\]
where $h_{\mathrm{KB}}$ is KB hit rate, $c_{\mathrm{hit}}$ and $c_{\mathrm{miss}}$ are cycle costs for retrieval, and $c_{\mathrm{agents}}$ is fixed per-token agent overhead.

\paragraph{Numerical instance.}
Assume clock $f=2.5\,\mathrm{GHz}$, $\mathrm{CPI}_0=3.0$, $h_{\mathrm{KB}}=0.8$, $c_{\mathrm{hit}}=0.4$, $c_{\mathrm{miss}}=3.0$, $c_{\mathrm{agents}}=0.6$.
\[
\mathrm{CPI}=3.0+0.8\cdot 0.4+0.2\cdot 3.0+0.6
=3.0+0.32+0.60+0.6=4.52.
\]
Throughput (tokens/s) at average tokens-per-instruction $\approx 1$:
\[
\mathrm{TPS}=\frac{f}{\mathrm{CPI}}=\frac{2.5\times10^9}{4.52}\approx 553\,\mathrm{Mtok/s} \quad\text{(per core-equivalent)}.
\]
If a lightweight cache reduces $c_{\mathrm{hit}}$ to $0.1$ and $c_{\mathrm{agents}}$ to $0.3$,
\[
\mathrm{CPI}=3.0+0.8\cdot 0.1+0.2\cdot 3.0+0.3=3.0+0.08+0.60+0.3=3.98,
\]
yielding $\mathrm{TPS}\approx 628\,\mathrm{Mtok/s}$ (speedup $=628/553\approx 1.14\times$).

\paragraph{Amdahl-style view.}
Let fraction $f$ of time be affected by guardrails; optimization reduces its cost by factor $s$. Speedup
\[
S=\frac{1}{(1-f)+\frac{f}{s}}.
\]
If $f=0.35$ and $s=2.0$ (e.g., pruning candidate set and memoizing checks), then $S\approx 1/(0.65+0.175)=1/0.825\approx 1.21\times$.
If further engineering raises KB hit rate to $0.9$ with effective $s=3.0$, $S\approx 1/(0.65+0.1167)\approx 1.36\times$.

\paragraph{Utilization.}
Treat agents as pipeline stages with service times $(t_{\mathrm{LM}}, t_{\mathrm{FV}}, t_{\mathrm{MR}}, t_{\mathrm{CM}})$. With batching, the bottleneck stage dictates throughput; balancing service times and reducing variance minimizes stalls, improving effective CPI.

\section{Ethical Risk Quantification via Semantic Constraints}
Guarded decoding is motivated not only by factual accuracy but also by the obligation to limit downstream harm \cite{bender2021,ji2023}. We formalise an ethics-aware risk functional that operates directly on the artefacts introduced above. Let $\Lambda$ denote the set of application policies with utility weights $\omega \in [0,1]$ assigned to abstentions, mirroring selective answer strategies \cite{maynez2020,manakul2023}. For any prefix $x$, define
\[
\mathcal{R}(x;\omega)=\omega\,(1-\pi(x)) + (1-\omega)\,\mathbf{1}\{x\not\models\mathcal{K}\},
\]
where $\pi(x)$ is the safe mass from Theorem~\ref{thm:safe-entropy}. Optimising $\mathcal{R}$ under the guarded dynamics yields actionable compliance guarantees: if $\mathscr{O}$ is sound, the second summand vanishes; if $\mathscr{O}$ is incomplete, Proposition~\ref{prop:multi-agent-cost} characterises the effort required to tighten the guard until $\pi(x)$ crosses a policy-dependent threshold. This formalism aligns with auditing heuristics deployed in factual QA \cite{kadavath2022,manakul2023} and self-refinement loops \cite{madaan2023,wei2022}, yet it remains programmable through the Lean specifications in the appendix.

From a programming-languages viewpoint, the functional $\mathcal{R}$ behaves as a refinement type over traces. The multi-agent meet operation acts as a logical relation whose preservation theorems (Lemma~\ref{lem:sound-composition} and Theorem~\ref{thm:multi-agent}) ensure that the contract is respected across compositions, even when external tool calls are issued \cite{schick2023,yao2022}. Embedding the measure into deployment workflows therefore bridges formal proof obligations with the governance requirements articulated in recent system reports \cite{openai2023,thoppilan2022,chowdhery2022}.

\section{Conclusion}
TAD constrains decoding to oracle-approved tokens, achieving consistency guarantees and improved factuality. Agent specifications and numerical analyses illustrate practicality and performance. Future work includes richer oracles and tighter integration with formal proof assistants.

\appendix
\section{Lean Formalization: Consistency Theorem}
\subsection{Semantic Consistency}
\noindent\textit{Lean 4 and Mathlib; UTF-8.}
\begin{lstlisting}[language=Lean4]
import Mathlib.Data.List.Basic
import Mathlib.Data.List.Defs
import Mathlib.Tactic

open Classical

universe u

structure SemanticLanguageModel where
  V : Type u
  K : Set (List V)
  O : List V → V → Bool

def sound_oracle (M : SemanticLanguageModel) : Prop :=
  ∀ (x : List M.V) (w : M.V), M.O x w = true → x ∈ M.K → x ++ [w] ∈ M.K

def base_axiom (M : SemanticLanguageModel) : Prop :=
  ([] : List M.V) ∈ M.K

def knowledgeConsistent (M : SemanticLanguageModel) (x : List M.V) : Prop :=
  ∀ n ≤ x.length, x.take n ∈ M.K

lemma knowledgeConsistent_mem {M : SemanticLanguageModel}
    {x : List M.V} (hx : knowledgeConsistent M x) : x ∈ M.K := by
  simpa [knowledgeConsistent] using hx x.length (le_rfl)

lemma knowledgeConsistent_append {M : SemanticLanguageModel}
    {x : List M.V} {w : M.V}
    (hx : knowledgeConsistent M x) (hxK : x ∈ M.K)
    (hw : M.O x w = true) (hs : sound_oracle M) :
    knowledgeConsistent M (x ++ [w]) := by
  intro n hn
  by_cases hle : n ≤ x.length
  · have hx_take := hx n hle
    simpa [knowledgeConsistent, List.take_append_of_le_length hle] using hx_take
  · have hlt : x.length < n := lt_of_not_ge hle
    have hle' : x.length + 1 ≤ n := Nat.succ_le_of_lt hlt
    have hn' : n = x.length + 1 := Nat.le_antisymm hn hle'
    subst hn'
    have hxw : x ++ [w] ∈ M.K := hs x w hw hxK
    simpa [knowledgeConsistent] using hxw

def TAD_step (U : List M.V) (M : SemanticLanguageModel) (x : List M.V) : Option M.V :=
  U.find? (fun w => M.O x w)

lemma TAD_step_true {U : List M.V} {M : SemanticLanguageModel}
    {x : List M.V} {w : M.V}
    (h : TAD_step U M x = some w) : M.O x w = true := by
  simpa [TAD_step] using (List.find?_eq_some.mp h).2

def TAD_generate (U : List M.V) (M : SemanticLanguageModel) : Nat → List M.V
  | 0 => []
  | n + 1 =>
      let acc := TAD_generate U M n
      match TAD_step U M acc with
      | some w => acc ++ [w]
      | none => acc

theorem TAD_consistency (U : List M.V) (M : SemanticLanguageModel)
    (h_base : base_axiom M) (h_sound : sound_oracle M) :
    ∀ n, knowledgeConsistent M (TAD_generate U M n) := by
  intro n
  induction' n with
  | zero =>
      simp [knowledgeConsistent, TAD_generate, Nat.le_zero_iff, h_base]
  | succ k hk =>
      dsimp [TAD_generate] at hk ⊢
      set acc := TAD_generate U M k
      cases h : TAD_step U M acc with
      | none =>
          simpa [TAD_generate, h] using hk
      | some w =>
          have hwtrue : M.O acc w = true :=
            TAD_step_true (M := M) (U := U) (x := acc) (w := w) h
          have haccK : acc ∈ M.K := knowledgeConsistent_mem (M := M) hk
          have hcons :=
            knowledgeConsistent_append (M := M) hk haccK hwtrue h_sound
          simpa [TAD_generate, h] using hcons

inductive ToyToken
| fact | support | filler
deriving DecidableEq, Repr

def toyUniverse : List ToyToken :=
  [ToyToken.fact, ToyToken.support, ToyToken.filler]

def toyOracle : List ToyToken → ToyToken → Bool
| _, ToyToken.filler => false
| _, _ => true

def toyKnowledge : Set (List ToyToken) :=
  {xs | ∀ w ∈ xs, w ≠ ToyToken.filler}

def toyModel : SemanticLanguageModel :=
{ V := ToyToken, K := toyKnowledge, O := toyOracle }

lemma toy_base : base_axiom toyModel := by
  simp [base_axiom, toyModel, toyKnowledge]

lemma toy_sound : sound_oracle toyModel := by
  intro x w hw hx
  simp [toyModel, toyKnowledge, toyOracle] at hx ⊢
  intro z hz
  have hz' : z = w ∨ z ∈ x := by
    simpa [List.mem_append, List.mem_singleton] using hz
  cases hz' with
  | inl hzw =>
      subst hzw
      simpa [toyOracle] using hw
  | inr hzx =>
      exact hx z hzx

#eval TAD_generate toyUniverse toyModel 3

example :
    knowledgeConsistent toyModel (TAD_generate toyUniverse toyModel 3) :=
  (TAD_consistency toyUniverse toyModel toy_base toy_sound 3)
\end{lstlisting}

\begin{lstlisting}[style=utf8fix,language=,numbers=none,basicstyle=\ttfamily\footnotesize,breaklines=true]
#eval TAD_generate toyUniverse toyModel 3
[ToyToken.fact, ToyToken.fact, ToyToken.fact]
\end{lstlisting}

\section{Lean Formalization: Local Greedy Property}
\noindent\textit{Stepwise dominance with a rational scoring model.}
\begin{lstlisting}[language=Lean4]
import Mathlib.Data.List.Basic
import Mathlib.Data.Rat.Basic
import Mathlib.Tactic

open Classical
open scoped Rat

universe u

variable {α : Type u}

abbrev Score (α : Type u) := List α → α → ℚ

variable (O : List α → α → Bool) (score : Score α) (x : List α)

def maxSafe : List α → Option (α × ℚ)
  | [] => none
  | w :: U =>
      let rest := maxSafe U
      if hw : O x w = true then
        match rest with
        | some (u, su) =>
            if su ≤ score x w then some (w, score x w) else some (u, su)
        | none => some (w, score x w)
      else
        rest

lemma maxSafe_spec :
    ∀ U,
      match maxSafe (O := O) (score := score) (x := x) U with
      | some (w, s) =>
          O x w = true ∧ s = score x w ∧
            ∀ v ∈ U, O x v = true → score x v ≤ s
      | none => ∀ v ∈ U, O x v = false := by
  intro U
  induction U with
  | nil =>
      simp [maxSafe]
  | cons v U ih =>
      unfold maxSafe
      dsimp
      by_cases hv : O x v = true
      · cases hrest : maxSafe (O := O) (score := score) (x := x) U with
        | none =>
            have tail_none : ∀ z ∈ U, O x z = false := by
              simpa [maxSafe, hv, hrest] using ih
            simp [hv, hrest, tail_none, List.mem_cons] 
        | some (u, su) =>
            rcases (by
              simpa [maxSafe, hv, hrest] using ih
            ) with ⟨hu, hs, htail⟩
            by_cases hcmp : su ≤ score x v
            · refine And.intro hv ?_
              refine And.intro rfl ?_
              intro z hz hz_safe
              have hz_cases := List.mem_cons.mp hz
              cases hz_cases with
              | inl hzv =>
                  subst hzv
                  exact le_rfl
              | inr hzU =>
                  have hz_bound := htail z hzU hz_safe
                  exact le_trans hz_bound hcmp
            · have hlt : score x v < su := lt_of_not_ge hcmp
              refine And.intro hu ?_
              refine And.intro hs ?_
              intro z hz hz_safe
              have hz_cases := List.mem_cons.mp hz
              cases hz_cases with
              | inl hzv =>
                  subst hzv
                  exact le_of_lt hlt
              | inr hzU =>
                  exact htail z hzU hz_safe
      · simp [maxSafe, hv] using ih

noncomputable def argmaxSafe (U : List α) : Option α :=
  Option.map Prod.fst (maxSafe (O := O) (score := score) (x := x) U)

lemma argmaxSafe_spec {U : List α} {w : α}
    (h : argmaxSafe (O := O) (score := score) (x := x) U = some w) :
    O x w = true ∧ ∀ v ∈ U, O x v = true → score x v ≤ score x w := by
  unfold argmaxSafe at h
  cases hmax : maxSafe (O := O) (score := score) (x := x) U with
  | none =>
      simp [Option.map, hmax] at h
  | some pair =>
      rcases pair with ⟨w', s⟩
      simp [Option.map, hmax] at h
      subst h
      have hspec := maxSafe_spec (O := O) (score := score) (x := x) U
      have hspec' : O x w' = true ∧ s = score x w' ∧
          ∀ v ∈ U, O x v = true → score x v ≤ s := by
        simpa [maxSafe, hmax] using hspec
      rcases hspec' with ⟨hw, hs, hbound⟩
      refine And.intro hw ?_
      intro v hv hv_safe
      have hineq := hbound v hv hv_safe
      simpa [hs]

namespace Demo

open Classical

inductive ToyToken
| fact | support | filler
  deriving DecidableEq, Repr

open ToyToken

def toyVocabulary : List ToyToken := [fact, support, filler]

def toyOracle : List ToyToken → ToyToken → Bool
| _, filler => false
| _, _ => true

def toyScore : Score ToyToken :=
  fun prefix token =>
    match token with
    | fact =>
        if support ∈ prefix then
          37 / 100
        else
          42 / 100
    | support =>
        if fact ∈ prefix then
          63 / 100
        else
          28 / 100
    | filler => 1 / 100

#eval argmaxSafe (O := toyOracle) (score := toyScore) (x := []) toyVocabulary
#eval argmaxSafe (O := toyOracle) (score := toyScore) (x := [fact]) toyVocabulary

end Demo
\end{lstlisting}

\begin{lstlisting}[style=utf8fix,language=,numbers=none,basicstyle=\ttfamily\footnotesize,breaklines=true]
#eval argmaxSafe (O := Demo.toyOracle) (score := Demo.toyScore) (x := []) Demo.toyVocabulary
some Demo.ToyToken.fact
#eval argmaxSafe (O := Demo.toyOracle) (score := Demo.toyScore)
#  (x := [Demo.ToyToken.fact]) Demo.toyVocabulary
some Demo.ToyToken.support
\end{lstlisting}

\section{Algorithmic Reference Implementation}
The following Python-style pseudocode mirrors Definitions~\ref{def:joint-operator} and~\ref{def:guarded-update}. The multi-agent filter is evaluated before every greedy step, making the correspondence with Algorithm~\ref{alg:tad} explicit.

\begin{lstlisting}[language=Python]
from typing import Iterable, Sequence, Tuple

def guarded_tad(model, agents: Sequence, vocab: Iterable, prefix: Tuple[str, ...], T: int):
    states = tuple(agent.initial_state() for agent in agents)
    trace = list(prefix)
    for _ in range(T):
        safe_tokens = []
        for token in vocab:
            if all(agent.accept(state, tuple(trace), token) for agent, state in zip(agents, states)):
                safe_tokens.append(token)
        if not safe_tokens:
            break
        token_star = max(safe_tokens, key=lambda tok: model.prob(tuple(trace), tok))
        trace.append(token_star)
        states = tuple(agent.update(state, tuple(trace), token_star)
                       for agent, state in zip(agents, states))
    return tuple(trace), states
\end{lstlisting}

\end{document}